\documentclass[conference,letterpaper]{IEEEtran}
\pdfoutput=1                                                          
\usepackage{fancyhdr}
\usepackage{tikz}                                                          
\usepackage{graphicx}
\usepackage{amssymb}
\usepackage{epstopdf}
\usepackage{graphics}
\usepackage{amsmath, amssymb}
\numberwithin{equation}{section}
\usepackage{graphicx}   
\usepackage{verbatim}   
\usepackage{color}      
\usepackage{subfig}  
\usepackage{hyperref}   
\usepackage{xspace}
\usepackage{float}
\usepackage[font={footnotesize}]{caption}

\newtheorem{defi}{Definition}
\newtheorem{theorem}{Theorem}[section]
\newtheorem{lemma}[theorem]{Lemma}

\DeclareMathOperator*{\argmax}{arg\,max}

\newcommand{\bff}{\mathbf{f}}
\newcommand{\bfr}{\bff^{R}}

\newcommand{\bfz}{\mathbf{z}}

\newcommand{\bs}[1]{\boldsymbol #1}

\newcommand{\bmu}{\bs\mu}

\newcommand{\bSigma}{\bs\Sigma}
\newcommand{\bsigma}{\bs\sigma}

\newcommand{\probot}{p(\bff^{R}\mid\bfz^{R}_{1:t})}
\newcommand{\phuman}{p(\bfh\mid\bfz^{h}_{1:t})}
\newcommand{\peye}{p(\bff^{i}\mid\bfz_{1:t}^{i})}

\newcommand{\pigpshort}{p(\bff^{R},\bff\mid \bar \bfz_{1:t})}

\newcommand{\bfh}{\mathbf{h}}

\newcommand{\fa}{p(\bfh, \bff^{R},\bff\mid\bfz_{1:t})}

\title{\large \bf
Assistive Planning in Complex, Dynamic Environments: a Probabilistic Approach}
  
\author{\IEEEauthorblockN{Pete Trautman}
\IEEEauthorblockA{Galois Inc., Portland OR\\
peter.trautman@galois.com}}

\date{}

\begin{document}

\maketitle
\thispagestyle{plain}
\pagestyle{plain}


\begin{abstract}
\noindent 
We explore the probabilistic foundations of shared control in complex dynamic environments. In order to do this, we formulate shared control as a random process and describe the joint distribution that governs its behavior.  For tractability, we model the relationships between the operator, autonomy, and crowd as an undirected graphical model.  Further, we introduce an interaction function between the operator and the robot, that we call ``agreeability''; in combination with the methods developed in~\cite{trautman-ijrr-2015}, we extend a cooperative collision avoidance autonomy to shared control.  We therefore quantify the notion of simultaneously optimizing over agreeability (between the operator and autonomy), and safety and efficiency in crowded environments.  We show that for a particular form of interaction function between the autonomy and the operator, linear blending is recovered exactly.  Additionally, to recover linear blending, unimodal restrictions must be placed on the models describing the operator and the autonomy.  In turn, these restrictions raise questions about the flexibility and applicability of the linear blending framework.  Additionally, we present an extension of linear blending called ``operator biased linear trajectory blending'' (which formalizes some recent approaches in linear blending such as~\cite{dragan-ijrr-2013}) and show that not only is this also a restrictive special case of our probabilistic approach, but more importantly, is statistically unsound, and thus, mathematically, unsuitable for implementation.  Instead, we suggest a statistically principled approach that guarantees data is used in a consistent manner, and show how this alternative approach converges to the full probabilistic framework.  We conclude by proving that, in general, linear blending is suboptimal with respect to the joint metric of agreeability, safety, and efficiency. 
  
\end{abstract}

\section{Introduction}
\label{sec:introduction}
\noindent Fusing human and machine capabilities has been an active research topic in computer science for decades.  In robotics related applications, this line of research is often referred to as \emph{shared autonomy}.  While shared autonomy has been addressed \emph{in toto}, it can be broken down into more specialized areas of research.  In particular, some researchers focus on the fusing of human and machine ``perception'' (see~\cite{morison-extend}).  Similarly, a great deal of research has focused on fusing human and machine ``decision making'' in the machine learning \cite{welinder-crowdsourcing,tamer}, control theory~\cite{pac-mdp-tlc}, and human robot interaction communities~\cite{goodrich-mixed-initiative}.

Under the umbrella of shared decision making, an even more focussed line of research has emerged: shared control, whereby the moment to moment control commands sent to the platform motors are a synthesis of human input and autonomy input.  Broadly speaking, shared control has been deployed in two cases: shared teleoperation (where the human operator is not co-located with the robot) and onboard shared control (where the human operator is physically on the robot).  Shared teleoperation is used for numerous applications: search and rescue~\cite{murphy-teleop} and extraterrestrial robotics~\cite{jpl-teleop} are two examples.

\begin{figure}[htbp]
  \centering
\includegraphics[scale=0.4]{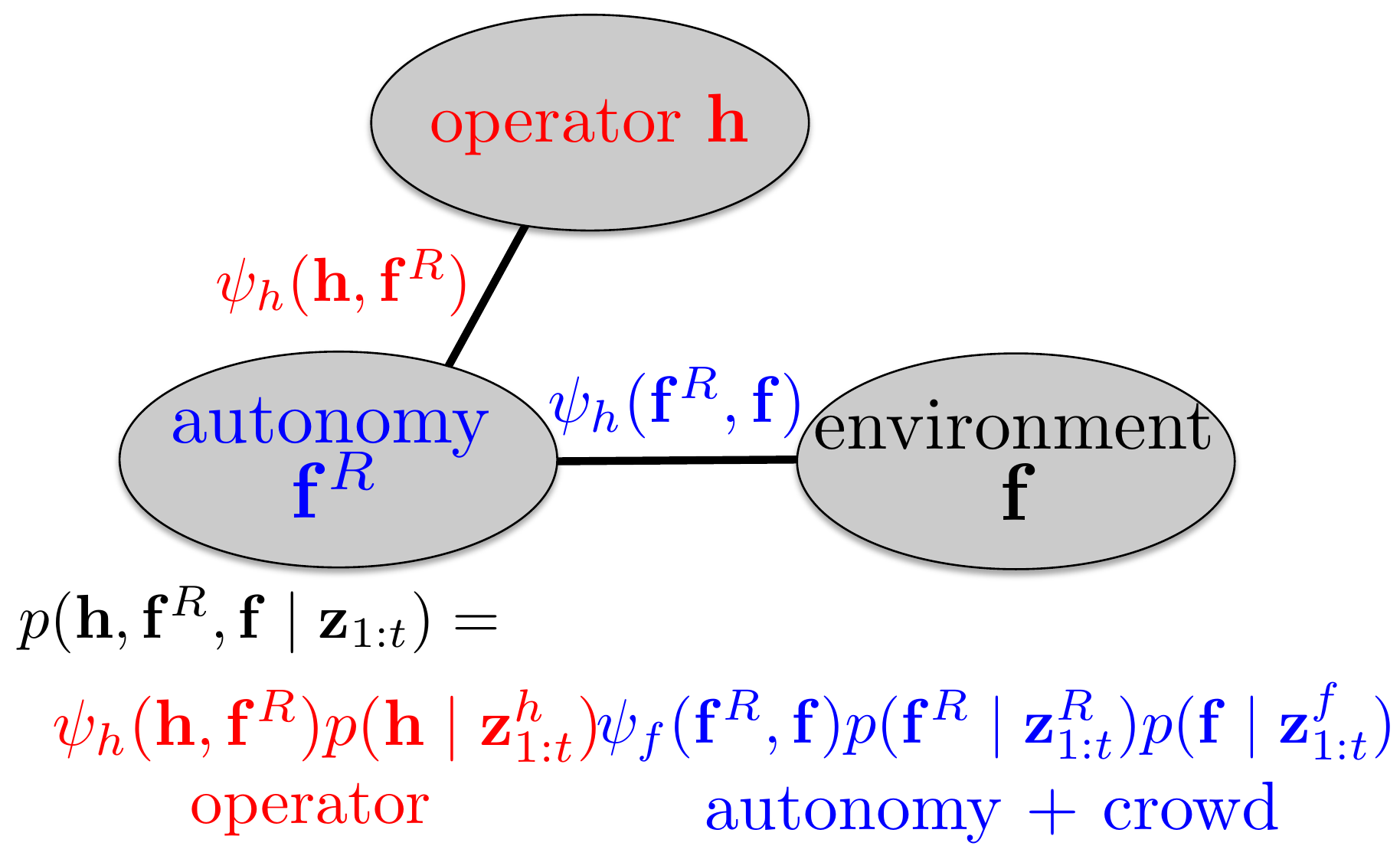}
  \caption{Diagram depicting the relationships of probabilistic shared control.}
  \label{fig:shared-autonomy-diagram}
  \vspace{-0mm}
\end{figure}

Onboard shared control has likewise enjoyed a wide variety of use cases: assistive wheelchair technology~\cite{kuipers-wheelchair}, assistive automobile driving, and assistive manufacturing vehicle operation (e.g., forklifts) are just a few of the examples.    For any of these cases, shared control can be broken down into an autonomous modeling step, a human modeling step, and a human-machine arbitration step (see \cite{dragan-ijrr-2013} for a compelling argument justifying this classification scheme).

In this paper, we explore the probabilistic foundations of onboard shared control in the presence of dynamic obstacles (``the crowd''): Figure~\ref{fig:shared-autonomy-diagram} depicts our approach. In order to do this, we formulate shared control as a random process and describe the joint distribution (over the random operator, autonomy, and crowd functions) that governs behavior and then propose a tractable model captures notions of operator-autonomy agreeability, safety, and efficiency; further, we prove that linear blending is a special case of this approach. In the course of the proof, we observe the following: first, in linear blending, the autonomy reasons \emph{independently} of the operator during optimization, so even if we had a precise operator model, the autonomy would not be informed of this information until after the optimization.  Second, in linear blending, the autonomy is limited to a single optimal decision, which is then averaged with the operator input---this approach leaves no flexibility in how the autonomy might assist the user.  Next, we present an extension of linear blending that can accommodate more than one statistic about the user (e.g., user inferred goal information, preferred trajectories, etc), and find that state of the art approaches to this problem are prone to statistical inconsistency.  We thus present a statistically valid model for how to properly condition the autonomy on user statistics.   We conclude with a section on the optimality (with respect to agreeability, safety, and efficiency) of linear blending and our probabilistic approach.

\section{Related Work}
\label{sec:lin-blend}
\noindent We begin by defining the arbitration step of linear blending:
\begin{align}
\label{eq:lin-blend}
u^s_{LB}(t) = K_hu^h_t + K_Ru^R_{t+1},
\end{align}
where, at time $t$, $u^s_{LB}(t)$ is the linearly blended shared control command sent to the platform actuators, $u^h_t$ is the human operator input (joystick deflections, keyboard inputs, etc.), $u^R_{t+1}$ is the next autonomy command,  and $K_h, K_R$ are the operator and autonomy \emph{arbitration functions}, respectively.  To ensure that the magnitude of $u^s_{LB}(t)$ does not exceed the magnitude of $u^h_t$ or $u^R_{t+1}$, we require that $K_h+K_R = 1.$  We observe that since $K_h = 1 - K_R$, there is only one free parameter in this formulation. 

This linear arbitration model has enjoyed wide adoption in the assistive wheelchair community (\cite{carlson-smc-2012,wang-adaptive-shared-control, lopes-embs-2010, urdiales-autonrobots-2011,yu-autonrobots-2003,peinado-icra-2011, urdiales-nsre-2013,inigo-blasco-isrrobotik-2014}).  Outside of the wheelchair community, shared control path planning researchers have widely adopted Equation~\ref{eq:lin-blend} as a \emph{de-facto} standard protocol, as extensively argued in \cite{dragan-ijrr-2013, draganrss2012} (in~\cite{dragan-ijrr-2013}, it is argued that ``linear policy blending can act as a common lens across a wide range of literature'').  Additionally, the work of \cite{poncela-smc-2009, wang-ras-2005} advocates the broad adoption of a linear arbitration step for shared control.

For the purposes of this article, we describe how each quantity of Equation~\ref{eq:lin-blend} is computed:

1) Collect the data at time $t$: $u^h_{1:t}$ are the historical operator inputs.   $\bfz^R_{1:t}$ are the historical measurements of the state of the robot (such as odometry, localization, SLAM output, etc).  

2) Compute the \emph{autonomous} input $u^R_{t+1}$:  This quantity may be computed using nearly any off the shelf planning algorithm, and is dependent on application.  The ``Dynamic Window Approach'' \cite{dynamic-window} and ``Vector Field Histograms$+$'' \cite{vfh+} are popular approaches to perform obstacle avoidance for wheelchairs.  Sometimes, the autonomy is biased according to data about the operator---for instance, one might imagine an offline training phase where the robot is taught ``how'' to move through the space, and then this data could be agglomerated using, e.g., inverse optimal control.  Alternatively, one might bias the autonomous decision making by conditioning the planner on the predicted or known human goal.


3) Compute the arbitration parameters $K_h$ and $K_R$.  A wide variety of heuristics have been adopted to compute this parameter: confidence in robot trajectory, smoothness, mitigating jerk, operator reliability, user desired trajectories, safeguarding against unsafe trajectories, etc.  Indeed, much of the shared control literature is devoted to developing novel heuristics to compute this parameter. 

4) Compute the shared control $u^s_{LB}(t)$ using Equation~\ref{eq:lin-blend}.

Typically, the data $u^h_{1:t}$ are interpreted literally---no likelihood or predictive model filters this data stream.  In other approaches, operator intention is modeled using a combination of dynamic Bayesian networks or Gaussian mixture models.

For this paper, we adopt the notation $\bfz^h_{t} \doteq u^h_t,$  (that is, we treat operator inputs as \emph{measurements} of the operator \emph{trajectory}, $\bfh \colon t\in\mathbb R \to \mathcal X$, where $\mathcal X$ is the action space).  Similarly, we define measurements $\bfz^R_{1:t}$ of the robot trajectory $\bfr$ and measurements $\bfz^i_{1:t}$ of the $i$'th static or dynamic obstacle trajectory $\bff^i$.  We thus work in the space of distributions over the operator function $\bfh$, autonomy function $\bfr$, and crowd function $\bff = (\bff^1, \ldots, \bff^{n_t})$, measured through $\bfz^f_{1:t}$.   The integer $n_t$ is the number of people in the crowd at time $t$.

We comment on the work of~\cite{leuven-compare-ml-map-pomdp,leuven-icra-2013, leuven-auro}.  In these papers, the authors construct a probabilistic model over user forward trajectories (i.e., a specific and personalized instantiation of $\phuman$).  They then formulate shared control as a \emph{partially observable Markov decision process} (POMDP), in order to capture the effects of robot actions on the probabilistic model of the operator.  However, for tractability of the POMDP, the autonomy is limited to only choosing from the next available state---this makes assistive navigation through crowds impossible (as shown in~\cite{trautmaniros}).  Further, the potential autonomous actions are limited by the input device; in one application, the autonomy was only able to reason over 9 directions.  As we discuss in Section~\ref{sec:operator-autonomy-disagreement}, limiting the autonomy to such a degree can negatively impact performance.
\vspace{-0mm}

\section{Foundations and Implementation of Probabilistic Shared Control}
\label{sec:psc-origins}
\vspace{-0mm}
\noindent As stated in Section~\ref{sec:introduction}, we seek to explore the probabilistic foundations of assistive onboard shared control; we thus posit
\begin{align}
\label{eq:probabilistic-shared-control}
\vspace{-3mm}
u^s_{PSC}(t) &= \bff^{R^*}_{t+1}\nonumber \\
 (\bfh,\bff^{R},\bff)^* &=\argmax_{\bfh, \bfr,\bff} p(\bfh, \bfr,\bff \mid \bfz^h_{1:t}, \bfz^R_{1:t},\bfz^f_{1:t}).
 \vspace{-2mm}
\end{align}
That is, probabilistic shared control (PSC) is the \emph{maximum a-posteriori} (MAP) value of the joint distribution over the operator, autonomy, and crowd.   We suggest this approach partly based on what we have learned from fully autonomous navigation in human crowds~\cite{trautman-ijrr-2015}---that the \emph{interaction} model between the robot and the human is the most important quantity---and partly based on the following: by formulating shared control as the MAP value of a joint probability distribution, we can explore the modeling limitations imposed by linear blending and the performance consequences of these limitations.  

In the remainder of this section, we explain how the approach in Equation~\ref{eq:probabilistic-shared-control} is a natural extension of our previous work in~\cite{trautmanicra2013}, and present a tractable model of the joint distribution over the operator, autonomy, and crowd.   The probabilistic graphical model in Figure~\ref{fig:pgm} guides our derivation.

\begin{defi}
A cooperative human crowd navigation model (as in~\cite{trautmanicra2013}) is described by
\begin{align}
\label{eq:igp}
\pigpshort = \psi(\bfr,\bff) \probot \prod_{i=1}^{n_t}\peye
\end{align}
where $\bar\bfz_{1:t} = [\bfz^R_{1:t}, \bfz^f_{1:t}]$ and $\probot, \peye$ are the robot and crowd individual dynamical prediction functions, respectively.  The robot-crowd interaction function $\psi(\bfr,\bff)$ captures joint collision avoidance: that is, how the robot and the crowd cooperatively move around one another so that the other can pass.  The graphical model is presented on the left hand side of Figure~\ref{fig:pgm}.
\end{defi}

\begin{figure}[htbp]
\vspace{-5mm}
  \centering
\includegraphics[scale=0.27]{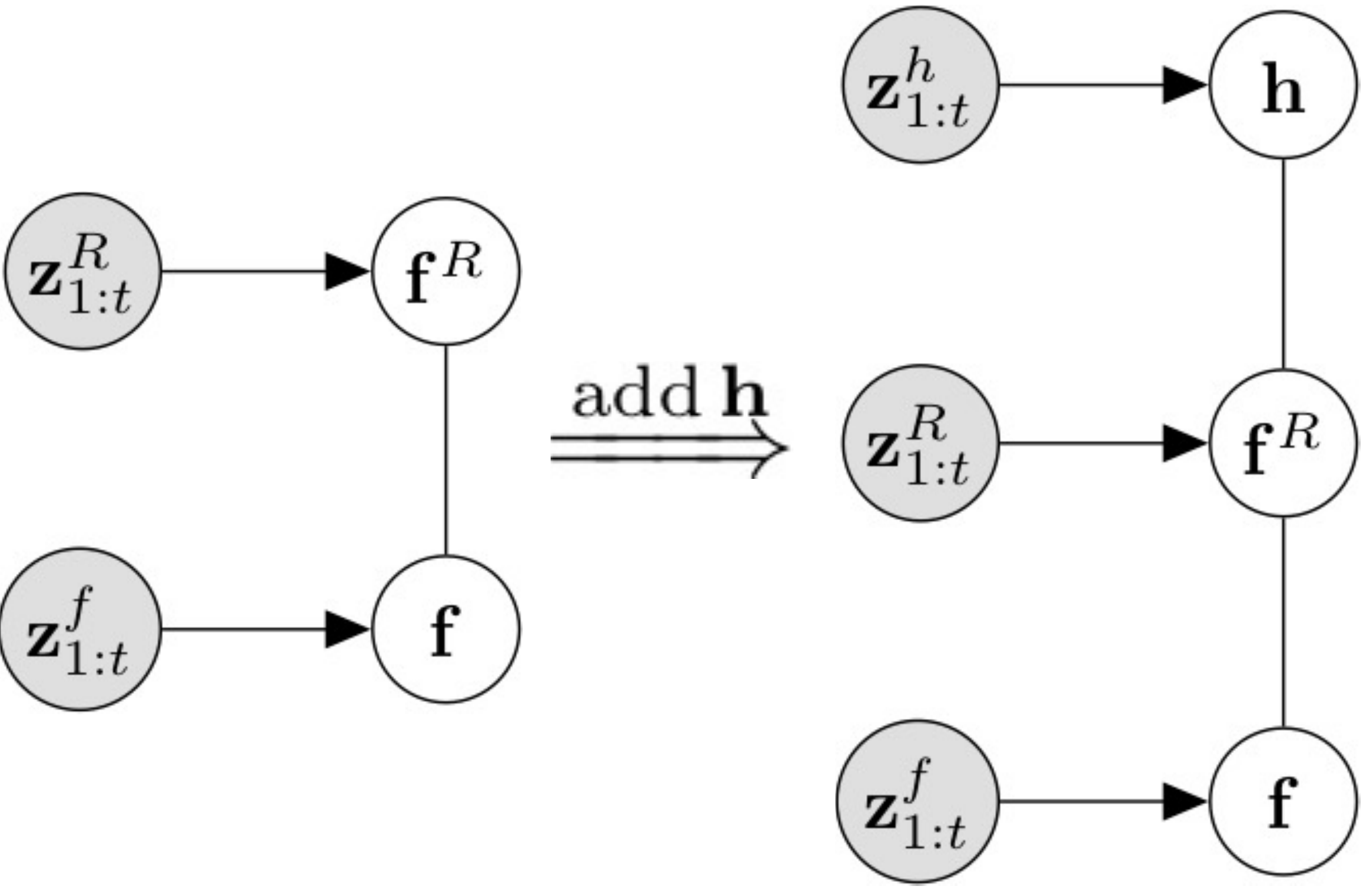}
  \caption{Correspondence between the cooperative model of Equation~\ref{eq:igp} (left hand side of figure) and probabilistic shared control in Equation~\ref{eq:obstacles}.}
  \label{fig:pgm}
  \vspace{-0mm}
\end{figure}

\begin{defi}
A probabilistic shared control (PSC) model in the presence of static or dynamic obstacles is
\begin{align}
\label{eq:obstacles}
\fa &= \psi(\bfh, \bfr) \phuman \pigpshort
\end{align}
where $\phuman$ is the predictive distribution over the operator, $\bfz^h_{1:t}$ is data generated by the operator, and $\psi(\bfh,\bfr)$ is the interaction function between the operator and the robot.  In analogy with~\cite{trautmanicra2013}, the interaction function $\psi(\bfh,\bfr)$ intends to capture how ``agreeable'' the robot trajectory is with the desires of the operator.  In this article, we choose $\psi(\bfh, \bfr) = \exp (-\frac{1}{2\gamma}(\bfh-\bfr)(\bfh-\bfr)^\top),$ where $\gamma$ captures how tightly the autonomy must couple to the operator; it is important to point out that many other reasonable functions could be used here.  Importantly, as shown in Theorem~\ref{thrm:psc-gen-lb}, this is the (implicit) operator-robot interaction function of linear blending.  The graphical model is presented on the right hand side of Figure~\ref{fig:pgm}.  
\end{defi}

The model in Equation~\ref{eq:obstacles} can be understood in the following way: the autonomy (modeled with $\pigpshort$) optimizes over safety and efficiency in crowds---specifically, the robot-crowd interaction function $\psi(\bfr,\bff)$ models how to move \emph{through} a crowd in the safest and most efficient way.  Similarly, the function $\psi(\bfh, \bfr)$ models how ``agreeable'' the robot path $\bfr$ is to the operator path $\bfh$.  By formulating $\bfh$ and $\bfr$ as random functions, flexibility of both parties is modeled, thus enabling a more realistic arbitration (where both parties are willing to compromise).  This idea was captured in~\cite{trautmaniros}: the robot and the crowd were treated as random functions so that they could come to a compromise about how to (cooperatively) pass by one another.

Thus, by choosing the joint $\argmax_{\bfh,\bfr,\bff}$ as in Equation~\ref{eq:probabilistic-shared-control}, we \emph{simultaneously} optimize over three quantities: agreeability, safety, and efficiency.   We suggest that such a joint optimization over safety and efficiency (the domain of the autonomy) and agreeability (adhering to the desires of the operator) is a reasonable way to formulate shared control.  Our analysis in the following sections provides credibility to this presumption.

\section{Linear Blending as Probabilistic Shared Control}
\label{sec:lb-as-psc}
\noindent To understand linear blending as a special case of Equation~\ref{eq:probabilistic-shared-control}, we consider the following conditioning relationship:
\begin{align}
\label{eq:h-condition}
\fa = p(\bfr, \bff \mid \bfz^R_{1:t}, \bfz^f_{1:t}, \bfh)\phuman.
\end{align}
To understand how the linear blending modeling assumptions effect the full joint, we first insert the linear blending \emph{operator} model: $\phuman = \delta(\bfh-\bfz^h_t)$.  Thus,
\begin{align*}
\fa &= p(\bfr, \bff \mid \bfz^f_{1:t}, \bfz^R_{1:t},\bfh) \delta(\bfh-\bfz^h_t) \\
&= p(\bfr, \bff \mid \bfz^f_{1:t}, \bfz^R_{1:t},\bfz^h_t). 
\end{align*}
In this case, 
\begin{align*}
\argmax_{\bfh,\bfr,\bff}\fa = \argmax_{\bfr,\bff}p(\bfr, \bff \mid \bfz^f_{1:t}, \bfz^R_{1:t},\bfz^h_t).
\end{align*}
Because this distribution is already conditioned on $\bfz^h_t$, there is no need for a linear arbitration step (we prove this in Theorem~\ref{thrm:psc-gen-lb}).  Further, by fixing the operator at $\bfz^h_t$, we are no longer \emph{jointly} optimizing over the autonomy and the operator.   

\begin{theorem}[\textbf{Equation~\ref{eq:probabilistic-shared-control} generalizes linear blending}]
\label{thrm:psc-gen-lb}
Let $p(\bfr, \bff \mid \bfz^f_{1:t}, \bfz^R_{1:t},\bfz^h_t) = \psi_h(\bfz^h_t,\bfr_{t+1})\pigpshort$
and use Laplace's Approximation~\cite{bishopbook} to approximate $\pigpshort$, where $\bar\bfr_{t+1}$ is a mode of the distribution:
\begin{align*}
\pigpshort = \mathcal N(\bfr_{t+1} \mid \bar\bfr_{t+1}, \bsigma^R).
\end{align*}
Then $u^s_{PSC}(t) = \bsigma( \frac{1}{\gamma}\bfz^h_t + \frac{1}{\bsigma^R}\bar\bfr)$,  $\bsigma^{-1} = (\gamma^{-1} + (\bsigma^R)^{-1})$, and $\gamma$ is the operator-autonomy attraction parameter.  
\end{theorem}

\begin{proof}
We compute
\begin{align}
\label{eq:decouple}
u^s_{PSC}(t) &= \argmax_{\bfh,\bfr,\bff} \fa \nonumber\\
& \propto \argmax_{\bfh,\bfr,\bff}\mathcal N(\bfr_{t+1} \mid \bfz^h_t, \gamma) \mathcal N(\bfr_{t+1} \mid \bar\bfr_{t+1}, \bsigma^R) \nonumber\\
& = \argmax_{\bfr_{t+1}} \mathcal N(\bfr_{t+1} \mid \bmu, \bsigma) \\
&=\bsigma( \frac{1}{\gamma}\bfz^h_t + \frac{1}{\bsigma^R}\bar\bfr )\nonumber
\end{align}
where $\bmu = \bsigma( \frac{1}{\gamma}\bfz^h_t + \frac{1}{\bsigma^R}\bar\bfr )$ and
$\bsigma^{-1} = (\gamma^{-1} + (\bsigma^R)^{-1}).$
\end{proof}
To make the correspondence with linear blending explicit, we choose $\bsigma^R = \gamma/K_R-\gamma$ and $\gamma = \bsigma^R/K_h - \bsigma^R$.  Then 
\begin{align*}
\bsigma( \frac{1}{\gamma}\bfz^h_t + \frac{1}{\bsigma^R}\bar\bfr ) &= \frac{\bsigma^R}{\bsigma^R + \gamma}\bfz^h_t + \frac{\gamma}{\bsigma^R+\gamma}\bar\bfr \\
&=K_h\bfz^h_t + K_R \bar\bfr.
\end{align*}
We recall that $K_h+K_R=1$ and so the mixing of $\bsigma_R$ and $\gamma$ is expected.  Only one free parameter exists in linear blending.  

We now examine the how the assumptions of linear blending can effect performance.  The linear blending restriction of the robot-crowd distribution to a unimodal Gaussian during arbitration $p(\bff^{R},\bff \mid \bfz^R_{1:t},\bfz^f_{1:t}) = \mathcal N(\bff^{R}_{t+1} \mid \bar\bfr_{t+1}, \bsigma^R_{t+1})$ can lead to severe restrictions in shared control capability; even though a locally optimal autonomy strategy $\bar\bfr$ is included in the linear blend $u^s_{LB}= K_h\bfz^h_t + K_R\bar\bfr_{t+1}$, there is no guarantee that the human operator will choose this optima, or even a nearby optimum (see Section~\ref{sec:operator-autonomy-disagreement}).   Alternatively, by maintaining the multitude of hypotheses inherent to $\pigpshort$, we greatly increase the possibility that the autonomy will be able to assist the operator in a way that is both desirable and safe.  As an example, linear trajectory blending over a \emph{safe} operator input and a \emph{safe} autonomous input can result in an unsafe shared trajectory.  That is, the weighted average of two safe trajectories can be averaged into an unsafe trajectory.

\section{Conditional Trajectory Blending}
\label{sec:ctb}
\noindent In this section, we extend our definition of linear blending to capture the salient characteristics of recent approaches such as~\cite{dragan-ijrr-2013}.  Broadly speaking, this line of work seeks to include additional data about the operator \emph{before} arbitration.  We show that this is statistically unfounded, and thus introduce a method (conditional trajectory blending) that incorporates operator information in a statistically principled way.

\begin{defi}[\textbf{Extending Equation~\ref{eq:lin-blend}}] 
Let

1) $\bar\bfh \sim p(\bfh \mid \bfz^h_{1:t})$ be defined through time $T>t$, such that $\bar\bfh \colon [1,T]\subset\mathbb R\mapsto \mathcal X,$

2) $\bar\bfr\sim p(\bfr, \bff \mid \bfz^R_{1:t},\bfz^f_{1:t})$, with $\bar\bfr \colon [1,T]\subset\mathbb R\mapsto \mathcal X,$

3) $u^s_{LTB}$ be the shared control $u^s_{LTB} \colon [1,T]\subset\mathbb R\to \mathcal X,$

4) and $K_h, K_R$ be the arbitration functions.
\end{defi}

\begin{defi}[\textbf{LTB}]
\label{def:ltb}
Linear Trajectory Blending (LTB) is defined as 1) Sample $\bar\bfr = \argmax_{\bfr,\bff} p(\bfr, \bff\mid \bfz^R_{1:t}, \bfz^f_{1:t})$, 2) Sample $\bar\bfh = \argmax_{\bfh} \phuman$, and 3) Construct the shared control $u^s_{LTB} =  K_h\bar\bfh +K_R \bar\bfr.$
\end{defi}
By extending to trajectories, we can more easily incorporate operator information into the arbitration step. 
In this vein, we now define an extension of linear trajectory blending that biases the autonomous decision making on operator data (motivated by the approach in~\cite{dragan-ijrr-2013}).

\begin{defi}[\textbf{LTBo}]
\label{def:operator-biased-ltb}
Let $p(G \mid \bfz^h)$ be a distribution about the operator, where $\bfz^h \subseteq \bfz^h_{1:t}$. Then \emph{operator biased linear trajectory blending} (LTBo) is defined as

1) Sample $G_1\sim p(G \mid \bfz^h)$.  $G_1$ could be a trajectory, a waypoint, a goal, or any other relevant quantity.

2) Sample the \emph{operator biased} distribution $$\bar\bfr_{\bfh} = \argmax_{\bfr,\bff} p(\bfr, \bff\mid \bfz^R_{1:t}, \bfz^f_{1:t}, G_1).$$

3) Sample the operator trajectory $\bar\bfh = \argmax_{\bfh} \phuman$.

4) Construct the shared control $u^s_{LTBo} = K_h\bar\bfh +K_R \bar\bfr_{\bfh}.$
\end{defi}
With these definitions, we extend the approach of Equation~\ref{eq:lin-blend} to include a model of the operator $\phuman$ and enable seeding the autonomy with operator information.  By modeling the operator with a distribution, we potentially bypass the issue of noisy inputs leading to ``jittery'' linear blends.  By seeding the autonomy with operator statistics, we might drive $u^s_{LTBo}$ towards solutions more closely aligned with user desire.  

While this approach is \emph{sensible}, no \emph{principles} motivate such an extension.  We thus pause to examine this approach in the context of the full joint distribution.  In particular, recall Equation~\ref{eq:h-condition}; if we only sample the single operator statistic $\bar\bfh = \argmax_{\bfh} \phuman$, then we are implicitly making the assumption that $\phuman = \delta(\bfh-\bar\bfh)$; thus we have
\begin{align*}
\fa &= p(\bfr, \bff \mid \bfz^f_{1:t}, \bfz^R_{1:t},\bar\bfh). 
\end{align*}
In this case, 
\begin{align*}
\argmax_{\bfh,\bfr,\bff}\fa = \argmax_{\bfr,\bff}p(\bfr, \bff \mid \bfz^f_{1:t}, \bfz^R_{1:t},\bar\bfh),
\end{align*}
and so (as before) there is no \emph{need} for a linear arbitration step---we have already conditioned the autonomy on the operator, and we can recover linear blending by choosing $p(\bfr, \bff \mid \bfz^f_{1:t}, \bfz^R_{1:t},\bar\bfh)$ as in Theorem~\ref{thrm:psc-gen-lb}.  If we have a \emph{separate} model $p(G \mid \bfz^h)$ about the operator, we know that it does not contain information beyond what is available in $\phuman$, since $\bfz^h \subseteq \bfz^h_{1:t}$ (see the discussion on the ``data processing inequality'' in~\cite{soatto-actionable}); worse, $p(G \mid \bfz^h)$ can lead to statistical inconsistencies (see Lemma~\ref{lem:ltb-unsound}).  We now show how to resolve this inconsistency and how to combine multiple operator data points in a statistically sound manner.

\begin{defi}[\textbf{Conditional Trajectory Blending}]
\label{def:conditional-arbitrary} Assume that we have the distributions $\phuman$ and $p(\bfr, \bff \mid \bfz^f_{1:t}, \bfz^R_{1:t},\bfh_b)$.  Then let $\{\bfh^b\}_{b=1}^{N_h} \sim \phuman$ be a collection of $N_h$ samples of the operator model.  If we take the model of the operator to be $\phuman = \sum_{b=1}^{N_h}w^b \delta(\bfh-\bfh^b),$ then
\begin{align}
\label{eq:ctb-arbitrary}
\fa &=p(\bfr, \bff \mid \bfz^f_{1:t}, \bfz^R_{1:t},\bfh)\sum_{b=1}^{N_h}w^b \delta(\bfh-\bfh^b) \nonumber \\
&=\sum_{b=1}^{N_h}w^b p(\bfr, \bff \mid \bfz^f_{1:t}, \bfz^R_{1:t},\bfh^b)
\end{align}
where $w^b = p(\bfh^b=\bfh \mid \bfz^h_{1:t})$ is the probability of sample $\bfh^b$.
We interpret the shared control to be $$u^s_{CTB} = \argmax_{ \bfr,\bff} \sum_{b=1}^{N_h}w^b p(\bfr, \bff \mid \bfz^f_{1:t}, \bfz^R_{1:t},\bfh^b).$$
\end{defi}
In particular, we revisit the case of LTBo: suppose that we sample $G_1\sim \phuman$---$G_1$ is present in this distribution since it contains all the data---and then sample $\bar\bfh\sim\phuman$.  Then conditional trajectory blending tells us that we should find the $\argmax$ of the distribution
\begin{multline*}
u^s_{CTB} = \argmax_{\bfr,\bff}\Big[w^1 p(\bfr, \bff \mid \bfz^f_{1:t}, \bfz^R_{1:t},G_1) +\\
 w^2 p(\bfr, \bff \mid \bfz^f_{1:t}, \bfz^R_{1:t},\bar\bfh)\Big].
\end{multline*}
In this case, then, $u^s_{CTB} \neq u^s_{LTBo}$; however, since conditional trajectory blending is derived directly from the full joint we know that it is combining the data $G_1$ and $\bar\bfh$ in a statistically sound manner. The next theorem provides information about the limiting behavior of conditional trajectory blending.
\begin{theorem}[\textbf{CTB approximates PSC}]
\label{thrm:cond-blend-igp}
As the number of operator samples tends to infinity, probabilistic shared control (Equation~\ref{eq:probabilistic-shared-control}) is recovered.
\end{theorem}

\begin{proof}
Representing $\phuman = \sum_{b=1}^{\infty} w^b \delta(\bfh -\bfh^b),$
\begin{align*}
&\fa  =p(\bfr, \bff \mid \bfz^R_{1:t}, \bfz^f_{1:t}, \bfh)\phuman \\
 &=p(\bfr, \bff \mid \bfz^R_{1:t}, \bfz^f_{1:t}, \bfh)\sum_{b=1}^{\infty} w^b \delta(\bfh -\bfh^b) \\
 &=\sum_{b=1}^{\infty}w^bp(\bfr, \bff \mid \bfz^R_{1:t}, \bfz^f_{1:t}, \bfh^b).
\end{align*}
\end{proof}

It is important to emphasize that conditioning the autonomy on operator statistics and then averaging $\bar\bfr_\bfh$ with a separate operator statistic $\bar\bfh$ is not just unnecessary, but potentially statistically unsound as well, since it is unclear how such an approach originates from Equation~\ref{eq:h-condition}.   

\begin{lemma}[\textbf{LTBo statistically unsound}]
\label{lem:ltb-unsound}
LTBo is not guaranteed to incorporate data in a statistically sound manner. 
\end{lemma}

\begin{proof}
Suppose that one were to sample $\bar\bfh \sim p(G\mid \bfz^h)$, then sample $\bar\bfh = \argmax_{\bfh} \phuman$, then compute $\bar \bfr_\bfh$, and then compute $K_h\bar\bfh +K_R \bar\bfr_\bfh$.  Since we have incorporated $\bar\bfh$ \emph{twice} in the linear blend, the data has been overused.  One could potentially compensate for this by ``removing'' the effect of double usage of $\bar\bfh$ in $K_h$, but it is unclear how to do this in a statistically sound manner.  

Even if $\bar\bfh$ is sampled twice in CTB, it is only counted once: $u^s_{CTB} = \argmax_{\bfr,\bff} \big[2w^1p(\bfr, \bff \mid \bfz^f_{1:t}, \bfz^R_{1:t},\bar\bfh)\big].$
\end{proof}

Thus, if we have a model of $p(\bfr, \bff \mid \bfz^R_{1:t}, \bfz^f_{1:t}, \bfh^b)$ and a model of the operator $\phuman,$ we have a clear mandate for how to correctly formulate shared control.

\section{Optimality of Shared Control}
\label{sec:operator-autonomy-disagreement}
\noindent We start with the following Gaussian sum approximations:
$\phuman = \sum_{m=1}^{N_h} \alpha_m\mathcal N(\bfh \mid \bmu_m, \bSigma_m)$
and $\pigpshort =\sum_{n=1}^{N_R} \beta_n \mathcal N(\bfr \mid \bmu_n, \bSigma_n).$
Then we have that
\begin{align}
\label{eq:gauss-sum}
&\psi_h(\bfh,\bff^{R})\phuman \pigpshort \nonumber\\
&=\psi_h(\bfh,\bff^{R}) \sum_{m=1}^{N_h} \alpha_m\mathcal N(\bfh \mid \bmu_m, \bSigma_m)\sum_{n=1}^{N_R}\beta_n \mathcal N(\bfr \mid \bmu_n, \bSigma_n),
\end{align}
and as $N_h, N_R \to \infty$, we recover the densities $\phuman$ and $\pigpshort$---see~\cite{gauss-sum}.  We note that the approximation $\sum_{n=1}^{N_R}\beta_n \mathcal N(\bfr \mid \bmu_n, \bSigma_n)$ consists of only safe modes, since $\psi_f(\bfr,\bff)$ will assign near zero $\beta$ to any unsafe modes.  This is not true for the operator Gaussian mixture, since the operator may make choices that place the platform on a collision course.  
\begin{theorem}[\textbf{LTB suboptimal}]
\label{thrm:ltb-suboptimal}
If either $\phuman$ or $\pigpshort$ is multimodal, then $u_{LTB}^s$ is suboptimal with respect to agreeability, safety, and efficiency.
\end{theorem}

\begin{proof}
Let $\bar\bfh = \argmax_{\bfh}\phuman$ and $\bar\bfr = \argmax_{\bfr,\bff}\pigpshort$.  Then we write Equation~\ref{eq:gauss-sum} as
\begin{multline*}
\psi_h(\bfh,\bff^{R}) \sum_{m=1}^{N_h} \alpha_m\mathcal N(\bfh \mid \bmu_m, \bSigma_m)\sum_{n=1}^{N_R}\beta_n \mathcal N(\bfr \mid \bmu_n, \bSigma_n) \\
= \psi_h(\bfh,\bff^{R}) \Big(\alpha^*\delta(\bfh-\bar\bfh)\beta^*\mathcal N(\bfr \mid \bar\bfr, \bSigma^*)+ \\
\sum_{m\neq *}^{N_h} \alpha_m\mathcal N(\bfh \mid \bmu_m, \bSigma_m)\sum_{n\neq *}^{N_R}\beta_n \mathcal N(\bfr \mid \bmu_n, \bSigma_n) \Big)
\end{multline*}
where we extracted the largest mixture components, and rewrote the summation using the notation $m \neq *$ to indicate that all the modes except the largest are present.  If $\phuman$ and $\pigpshort$ are unimodal, then we recover linear blending as the optimally agreeable, safe, and efficient solution (Theorem~\ref{thrm:psc-gen-lb}).   In general, however, the $\argmax_{\bfh,\bfr,\bff}$ does not correspond to the linear blend.     

We point out that this result is \emph{independent of how $\psi(\bfh,\bfr)$ is chosen}.  In other words, for any definition of ``agreeability'', the suboptimality of linear blending still holds. 
\end{proof}  

\begin{figure}[h]
  \centering
\includegraphics[scale=0.36]{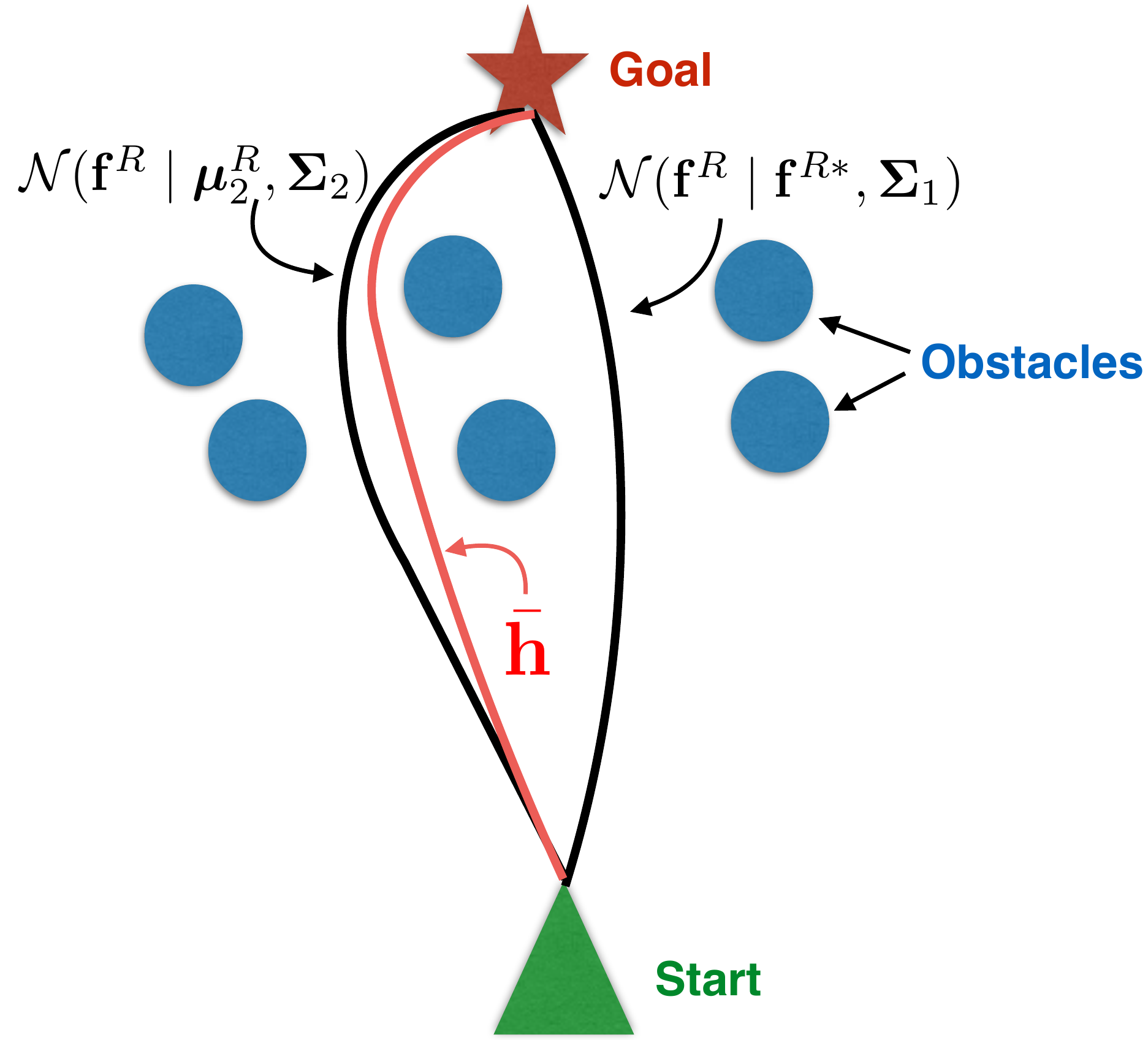}
  \caption{One global autonomy optima at $\bff^{R*}$ and a safe but suboptimal autonomy mode at $\bmu^R_2$ through some obstacle field (additional autonomous modes exist but we leave them off for clarity).  The operator's unimodal predicted trajectory at $\bar\bfh$ is \emph{safe}.  Covariance functions removed for clarity.}
  \label{fig:two-mode-autonomy}
\end{figure}
To understand why this is important, first consider the illustration in Figure~\ref{fig:two-mode-autonomy}.  Suppose that there are two autonomous safe modes through the obstacle field: $\mathcal N(\bfr \mid \bff^{R*}, \bSigma_1)$ and $\mathcal N(\bfr \mid \bmu^R_2, \bSigma_2).$  Thus
\begin{align*}
&\psi_h(\bfh,\bff^{R})\phuman \pigpshort \\
&=\psi_h(\bar\bfh,\bff^{R})[ \beta_1\mathcal N(\bfr \mid \bff^{R*}, \bSigma_1) +  \beta_2\mathcal N(\bfr \mid \bmu_2^R, \bSigma_2)] \\
&= \frac{\beta_1}{Z_1}\mathcal N(\bfr \mid \bar\bff^{R*}, \bar\bSigma_1) +  \frac{\beta_2}{Z_2}\mathcal N(\bfr \mid \bar\bmu_2^R, \bar\bSigma_2)
\end{align*}
where $\beta_1>\beta_2$ (the first mode is the global optima). However
\begin{align*}
\frac{1}{Z_1} &\propto \exp\left( -\frac{1}{2} \left(\bar\bfh - \bff^{R*}\right)^{\top}
\left(\gamma + \bSigma_1\right)^{-1}\left(\bar\bfh - \bff^{R*}\right) \right)\\
\frac{1}{Z_2} &\propto \exp\left( -\frac{1}{2} \left(\bar\bfh - \bmu_2\right)^{\top} 
\left(\gamma + \bSigma_2\right)^{-1}\left(\bar\bfh - \bmu_2\right) \right),
\end{align*}
and so $1/Z_1$ is exponentially smaller than $1/Z_2$ since $\bar\bfh - \bff^{R*}$ is much larger than $\bar\bfh - \bmu_2.$  Thus, the probabilistic shared control in this situation is very close to both $\bmu^R_2$ and $\bar\bfh.$

Conversely, $u^s_{LTB} = K_h\bar\bfh + K_R\bff^{R*}.$  Here, if $K_h$ is close to $K_R,$ then $u^s_{LTB}$ is unsafe.  If $K_R \gg K_h,$ then the autonomy overrides the operator's (safe) choice.   If $K_h \gg K_R$, then the operator is controlling the platform, and is thus not being assisted.  This begs the following question: what heuristic should be employed to choose $K_h$ and $K_R$?  By carrying multiple modes (as in probabilistic shared control), \emph{we bypass this dilemma}, since heuristics are never invoked.  Basic rules of probability theory (namely, the normalizing factor) determine the best choice.  In other words, probabilistic shared control is able to determine the shared control in a data driven way rather than through anecdote.

\begin{figure}[h]
  \centering
\includegraphics[scale=0.36]{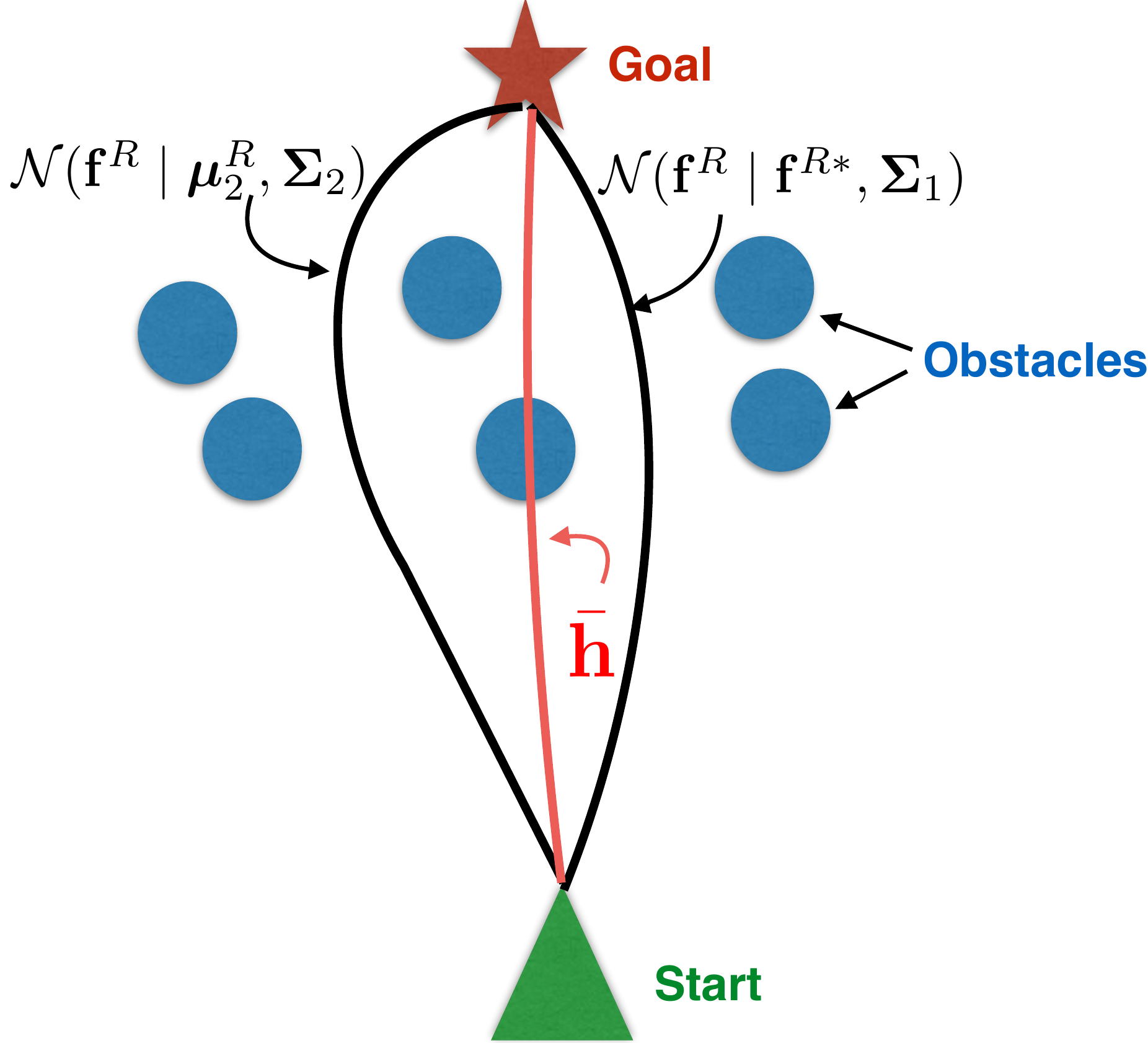}
  \caption{One global autonomy optima at $\bff^{R*}$ and a safe but suboptimal autonomy mode at $\bmu^R_2$ through some obstacle field (additional autonomous modes exist but we leave them off for clarity).  The operator's unimodal predicted trajectory at $\bar\bfh$ is \emph{unsafe}.  Covariance functions removed for clarity.  }
  \label{fig:unsafe-operator}
  \vspace{-0mm}
\end{figure}
In Figure~\ref{fig:unsafe-operator} the operator chooses an \emph{unsafe} trajectory.  For linear trajectory blending, $u^s_{LTB} = K_h\bar\bfh + K_R\bff^{R*}$, and so we must choose our heuristics such that $K_R \gg K_h$ in order to avoid collision---that is, we must insert logic that overrides the operator when the operator makes unsafe decisions.

In contrast, consider
\begin{align*}
&\psi_h(\bfh,\bff^{R})\phuman \pigpshort 
=\psi_h(\bfh,\bff^{R}) \phuman \times\\
&[ \beta_1\mathcal N(\bfr \mid \bff^{R*}, \bSigma_1) +  \beta_2\mathcal N(\bfr \mid \bmu_2^R, \bSigma_2)] \\
&=\psi_h(\bfh,\bff^{R}) \mathcal N(\bfh \mid \bmu_h, \bSigma_h) \times\\
&[ \beta_1\mathcal N(\bfr \mid \bff^{R*}, \bSigma_1) +  \beta_2\mathcal N(\bfr \mid \bmu_2^R, \bSigma_2)]
\end{align*} 
where we have maintained both autonomous modes and a unimodal distribution over $\bfh$.  In this situation $\beta_1$ is close to $\beta_2$, and the difference between the operator mean and either of the autonomy means are nearly the same, so $Z_1$ is close to $Z_2$.  However, both $\bSigma_1$ and $\bSigma_2$ are both fairly narrow (otherwise, they are not safe modes), and because there is flexibility in $\phuman$, the MAP value is close to either $\bff^{R*}$ or $\bmu^R_2$.  Because the operator is treated probabilistically, heuristics are not used to detect poor operator choices: the autonomy assists the operator by blending near the tails of $\phuman$.  

\begin{figure}[h]
  \centering
\includegraphics[scale=0.36]{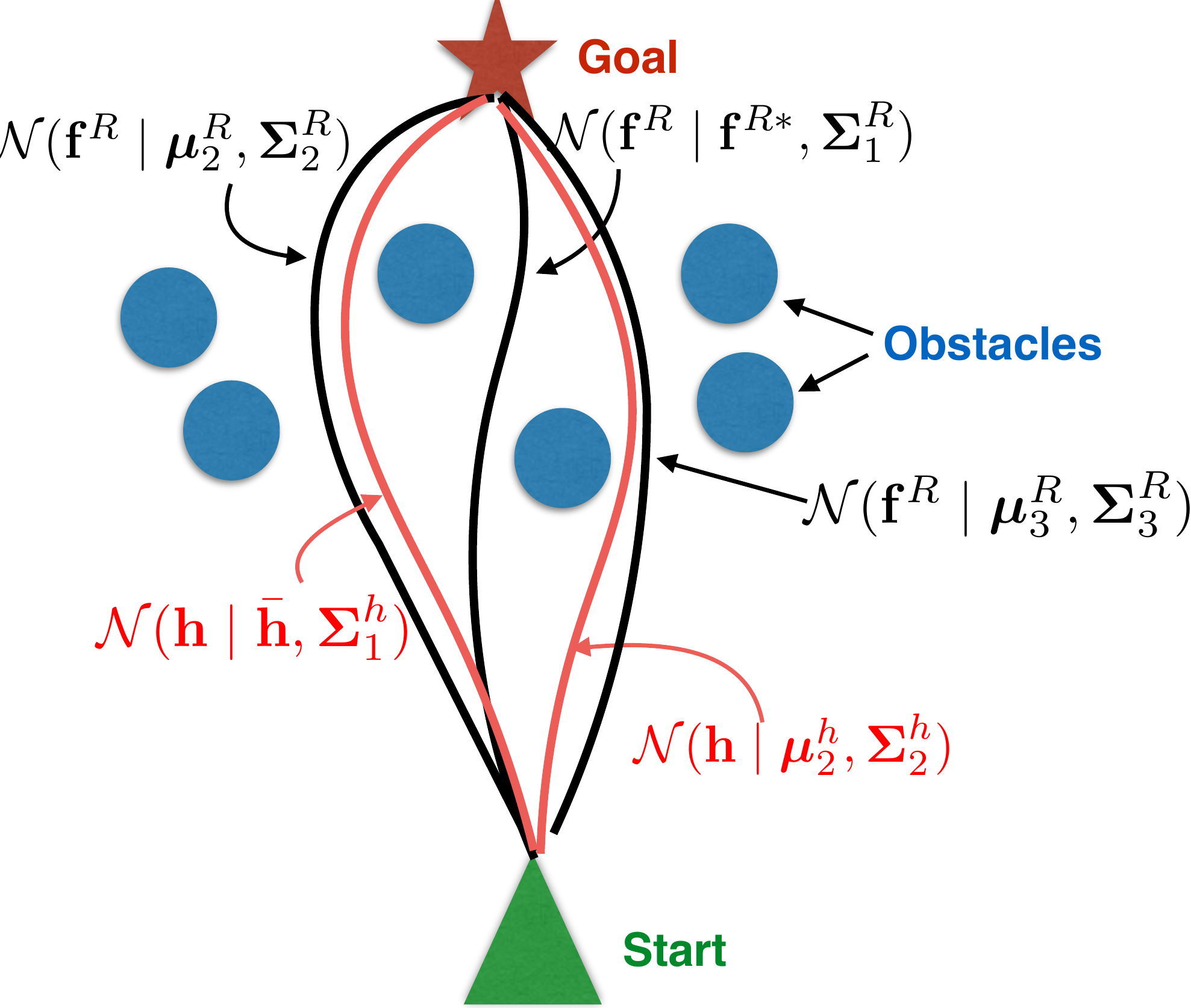}
  \caption{One global autonomy optima at $\bff^{R*}$ and two safe but suboptimal autonomy modes at $\bmu^R_2$ and $\bmu^R_3$ through some obstacle field.  The operator's bimodal predicted trajectory centered at $\bar\bfh$ and $\bmu_2^h$.  Covariance functions removed for clarity.}
  \label{fig:multimodal-operator}
  \vspace{-0mm}
\end{figure}

In Figure~\ref{fig:multimodal-operator}, the operator generates ambiguous data about how he wishes to move between the start and the goal, and so $\phuman$ is bimodal.  Linear blending produces $u^s_{LTB} = K_h\bar\bfh + K_R\bff^{R*}$ and so we end up with a situation very similar to that discussed in Figure~\ref{fig:two-mode-autonomy}---an autonomy and an operator that are needlessly in disagreement, and thus difficult to disambiguate with $K_h$ and $K_R$.  To be fair, if $\bar\bfh$ and $\bff^{R*}$ happen to lie close to one another, then linear blending provides the optimal solution.  But for multimodal autonomous and operator distributions, such a situation is the exception rather than the rule (this exception requiring that the operator make globally optimal decisions).

For probabilistic shared control, the shared control is likely a trajectory near $\bar\bfh$ and $\bmu^R_2$---a solution that is both safe and respects the operator's desires (although, depending on the weights $\alpha_m$ and $\beta_n$ the solution may be near $\bmu^h_2$ and $\bmu_3^R$, a solution that reflects the operator's desires and is still safe).

These figures illustrate Theorem~\ref{thrm:ltb-suboptimal}: that linear blending is suboptimal when either the operator or the autonomy model is multimodal; anecdotally, this is a result of restricting the blend to a single autonomous decision and a single operator trajectory, which can be in conflict when nontrivial modes are not reasoned over.  
\vspace{-0mm}

\section{Conclusions}
\vspace{-0mm}
\noindent We presented a probabilistic formalism for shared control, and showed how the state of the art in shared control, linear blending, and a trajectory based generalization, are special cases of our probabilistic approach.  Further, we showed that linear blending is prone to statistical inconsistencies, but that the probabilistic approach is not.  Finally, we proved that linear blending is suboptimal with respect to agreeability, safety and efficiency.  Importantly, no experiments were performed, and so questions remain as to the viability of this approach.

\bibliographystyle{abbrv}
{\footnotesize
\bibliography{../standard_bibliography}
}

\end{document}